\documentclass{article}
\usepackage{geometry} 

\RequirePackage{amsthm,amsmath,amssymb}
\usepackage{algorithm}
\usepackage{algorithmic}
\usepackage{ulem}
\usepackage{multirow}
\usepackage{booktabs}
\usepackage{natbib}
\RequirePackage[colorlinks,citecolor=blue,linkcolor=blue,urlcolor=blue,pagebackref]{hyperref}
\usepackage{xr}
\usepackage{comment}
\usepackage{subcaption} 
\usepackage{graphicx} 
\usepackage{float}      % 用于 [H] 强制位置
\usepackage{authblk}

\def\argmin{\mathop \text{\rm{argmin}}}
\theoremstyle{plain}
\newtheorem{theorem}{Theorem}[section]

\newtheorem{corollary}[theorem]{Corollary}
\theoremstyle{definition}

\newtheorem{assumption}[theorem]{Assumption}
\theoremstyle{remark}
\newtheorem{remark}[theorem]{Remark}

\title{Deep Optimal Individualized Treatment Rules for Bivariate Survival Outcomes via Adaptive Prediction-Powered Learning}

\author[1]{Kun Ren}
\author[2]{Yifan Cui} 
\author[1]{Wen Su}

\affil[1]{Department of Biostatistics, City University of Hong Kong}
\affil[2]{Center for Data Science, Zhejiang University} 

\date{}
\begin{document}
\maketitle

\begin{abstract}
In randomized trials involving multiple treatments, bivariate survival outcomes present significant analytical challenges for making decisions. This paper addresses the problem of deriving optimal individualized treatment rules to maximize the joint survival probability beyond fixed time points $(t_1, t_2)$ through deep neural networks, while accounting for right censoring. We propose a novel approach that models treatment rules via stochastic policies, coupling marginal accelerated failure time models via link function to capture bivariate dependence. To enhance robustness and effectiveness of decision making, we introduce an adaptive prediction-powered method that leverages auxiliary predictions from machine learning models.  
\end{abstract}

\section{Introduction}

In investigations of precision medicine, the estimation of optimal Individualized Treatment Rules (ITRs) has become a cornerstone of research, aiming to offer reasonable medical decisions to patient heterogeneity based on baseline covariates \cite{zhao2012estimating,shi2018high,wang2018learning}. 
While significant progress has been made in developing ITRs for univariate outcomes, real-world clinical studies often present a more complex picture. 
For example, randomized clinical trials frequently involve multiple treatments with bivariate survival outcomes. 
In these settings, the collected data usually consist of paired event times that may be right-censored and exhibit strong dependence \cite{shih1995inferences,marra2020copula,he2024analysis}. 
In such complex scenarios, developing the optimal ITR becomes extremely difficult, while existing studies for ITR typically focus on univariate outcomes \cite{jiang2017estimation,cui2023estimating,cui2024learning}.

To fill the research gap, we propose a novel method for deriving optimal ITRs that maximize the joint survival probability at fixed time points $(t_1, t_2)$. 
We model the marginal distributions of bivariate survival times by accelerated failure time (AFT) models \cite{lin1998accelerated,zeng2007efficient} and capture the dependence of them via copula links. 
Different from deterministic decision rules which are often non-differentiable and hard to optimize, we formulate the ITR as a stochastic policy parameterized by deep neural networks (DNNs).

A central innovation of our work is the integration of modern machine learning predictions into statistical learning to improve the accuracy and robustness of decision making.
We introduce an \textbf{adaptive prediction-powered (APP)} estimation method. 
Leveraging auxiliary information from black-box machine learning models, the proposed APP estimator can be designed to be robust even if machine learning models are unreliable.
We also offer a corresponding algorithm to ensure computational effectiveness.

The main contributions of this paper are summarized as follows:
\begin{itemize}
\item \textbf{A Unified Bivariate Survival ITR Framework:} 
We propose a flexible approach that accommodates multiple treatments and bivariate right-censored time-to-event data. 
By integrating AFT marginal models with a copula-based structural equation, the proposed method effectively captures the latent correlation between survival times, which is critical for optimal decision making.

\item \textbf{Adaptive Prediction-Powered Estimation:} 
We develop a novel estimation method that incorporates machine learning predictions into survival analysis. 
Extending concepts from prediction-powered inference \citep{angelopoulos2023prediction, zrnic2024cross}, the proposed APP estimation method allows for the flexible adjustment of weights based on the quality of the predictor. This approach ensures rigorous statistical guarantees while significantly boosting estimation efficiency compared to classical methods.

\item \textbf{Computational Efficiency:} 
We reformulate the optimal decision problem as learning a stochastic policy distribution via a softmax-output deep neural network. 
Such a procedure transforms the typically non-convex, non-smooth optimization problem associated with indicator-based ITRs into a smooth objective function. 
This allows the application of efficient gradient-based optimization algorithms and ensures computational feasibility.

\item \textbf{Theoretical Guarantees:} 
We provide a comprehensive theoretical analysis, establishing the consistency of the proposed APP estimator and the risk consistency of the derived deep ITR. 
\end{itemize}

\section{Related work}

\textbf{Optimal Individualized Treatment Rules in Survival Analysis} 

The estimation of ITRs has been extensively studied for continuous and binary outcomes. 
For right-censored time-to-event data, \cite{goldberg2012q, zhao2015doubly, bai2017optimal} provided several weighting methods to estimate expected survival time or survival probability with decison making.
\cite{jiang2017estimation} proposed methods to maximize $t$-year survival probabilities using inverse probability weighting.  

Our study extends the above works by addressing the specific complexities of \textit{bivariate} survival outcomes, where the optimization objective involves the conditional joint probability $P(T_1(a) > t_1^*, T_2(a) > t_2^* \mid X=x)$, requiring explicit modeling of the dependence structure between potential outcomes $T_1(a)$ and $T_2(a)$ under treatment level $a$. 

\textbf{Bivariate Survival Modeling and Copulas} 

Modeling correlated failure times under censoring is a classical problem in survival analysis.  
To capture dependence, copula models have gained popularity due to their flexibility in separating marginal distributions from the dependence structure. 
\cite{marra2020copula} demonstrated the efficacy of copula link-based additive models for right-censored event time data. 

We adopt the above methodology, utilizing copulas to link marginal survival models, thereby allowing for the derivation of a precise objective function for the target decision learning task.

\textbf{Prediction-Powered Inference} 

Recent advances in statistical inference have focused on utilizing powerful but potentially biased machine learning predictions to improve estimation. 
\cite{angelopoulos2023prediction} introduced prediction-powered inference (PPI) to construct valid confidence intervals using black-box predictors in semi-supervised settings. 
Similarly, \cite{zrnic2024cross} explored cross-prediction-powered inference to improve the estimation efficiency.  

Our contribution is on adapting these concepts specifically for the estimation of nuisance functions in complex survival models. 
Different from standard semi-supervised settings, in the presence of right censoring, the proposed APP estimation method uses stable predict machines to reduce the variance of the \textit{inverse-probability-of-censoring weighting} (IPCW) least squares estimator \citep{li2003empirical}. 
By introducing a user-specified weight $c$, our method generalizes existing debiasing procedures, offering a flexible trade-off between bias correction and variance reduction.

\textbf{Deep Learning for Decision Making}

The integration of DNNs into decision-making frameworks has significantly advanced the estimation of optimal ITRs. 
Different from traditional approaches that separate outcome modeling from decision assignment, recent advances utilize deep learning to handle high-dimensional heterogeneity and complex decision boundaries. 
For example, \cite{saleh2024skin} utilized neural networks to capture the non-linear relationships between variables and optimize decision rules, demonstrating the flexibility of deep architectures in constrained optimization settings.

We propose a novel deep policy learning approach in which DNNs play a dual role: they serve as black-box predictive models for imputing censored outcomes and nonparametrically model the optimal policy distributions.
By employing a softmax transformation, we ensure the learned policy is probabilistic and everywhere differentiable. 
This formulation relaxes the discrete treatment selection problem, smoothing the objective function and facilitating stable training via backpropagation, which offers a substantial computational advantage over the direct optimization of non-smooth, discrete decision rules.

\section{Preliminary}  
\subsection{Model and assumption}
Consider a randomized trial with $K+1$ treatment levels indexed by $a \in \mathcal{A} := \{0, \dots, K\}$. Let $X \in \mathcal{X} \subset \mathbb{R}^p$ denote the vector of baseline covariates. 
We are interested in the bivariate survival times $(T_1, T_2) \in \mathbb{R}^2$. Under the potential outcomes framework, let $(T_1(a), T_2(a))$ denote the potential event times under treatment $a$.

Due to right censoring, we observe $(Y_{1}, Y_{2}) = (\min(T_{1}, C_{1}), \min(T_{2}, C_{2}))$ and indicators $(\Delta_{1}, \Delta_{2}) = (I(T_{1} \le C_{1}), I(T_{2} \le C_{2}))$, where $(C_1, C_2)$ are censoring times. 
The observed data consists of $n$ i.i.d. samples $\mathcal{O} = \{O_i\}_{i=1}^n= \{Y_{1i}, Y_{2i}, \Delta_{1i}, \Delta_{2i}, X_i, A_i\}_{i=1}^n$, and the random vector $O$ has probability law $P$. 
Let $n_a$ be the sample size of grouped data $\mathcal{O}_{a} = \{(X_i , Y_{1i},Y_{2i}, \Delta_{1i},\Delta_{2i})\}_{i=1}^{n_a}$ under treatment $a\in\mathcal{A}$. 

We consider following standard causal inference assumptions \cite{imbens2004nonparametric} under the randomized controlled trial:
\begin{assumption} \label{as1}
For any $j=1,2$ and $a\in\mathcal{A}$,

(i) Consistency: $T_j = \sum_{a=0}^K T_j(a) I(A=a)$;

(ii) Unconfoundedness: $(T_1(a), T_2(a)) \perp A \mid X$;

(iii) Positivity: $ c < P(A=a) < 1-c$ with a constant $c\in(0,1)$.
\end{assumption}

In this study, we employ the accelerated failure time model for marginal outcomes, for $a\in\mathcal{A}$,
\begin{equation} \label{AFT}
\log T_j(a) = \beta_{ja}^\top X + \epsilon_{ja}, \quad j=1,2,
\end{equation}
where $\beta_{ja}\in\mathbb{R}^{p}$ captures the heterogeneous effects of covariates on potential outcomes varying with different treatment levels. 
Suppose that $\epsilon_{ja}$ follows the a normal distribution with unknown variance $\gamma_{ja}^2$.
Therefore, the \textit{conditional marginal survival function} of $T_j(a)$ given $X=x$ can be parametrically modeled as
\begin{align*}
&S_j(t ,x,a; (\beta_{ja},\gamma_{ja})) 
= P( U > (\log t - \beta_{ja}^{\top} X) / \gamma_{ja} \mid X=x),
\end{align*} 
where the independent random variable $U$ follows the standard normal distribution. 

Moreover, to identify the joint distribution of survival times through associated marginal distributions, we tend to apply the method of copula link-based modeling \citep{Marra02042020}. 
In particular, let the \textit{conditional joint survival function} of $(T_1(a),T_2(a))$ given $X=x$ at time point $(t_1, t_2)$ be 
\begin{align*}
& S(t_1, t_2, x, a; \eta_a) 
=  L(S_1(t_1, x, a; (\beta_{1a},\gamma_{1a})), S_2(t_2, x, a; (\beta_{2a},\gamma_{2a})); \theta_a),
\end{align*} 
where  the link function $L(\cdot,\cdot,\theta) : (0, 1)^2 \to (0, 1)$ is parameterized by $\theta \in \mathbb{R}$, and 

$\eta_a := (\beta_{1a}, \beta_{2a}, \gamma_{1a}, \gamma_{2a}, \theta_a)$ represents all unknown parameters.
Since the true relationship between survival functions is unknown in practice, we employ cross-validation to empirically select the optimal link function from the candidates summarized in \cite{Marra02042020}. 
Specifically, the selection criterion identifies the link function that yields the lowest empirical risk.

Under Assumptions \ref{as1} (i) and (ii), the conditional survival functions of the potential outcomes $T_j(a)$ can be expressed as those of the observed responses $T_j$ given $X$ and treatment $A=a$. 
Consequently, these counterfactual quantities are identifiable from the observed data.

\subsection{Optimal individualized treatment rule}
Let $D: \mathcal{X} \to \mathcal{A}$ denote the treatment rule, which assigns treatment based on patient characteristics. 
Our goal is to identify an optimal individualized treatment rule $D_0$ that maximizes the expected survival probability at $(t_1^*,t_2^*)$, that is,
$$\max_D E[P(T_1(D) > t_1^*, T_2(D) > t_2^* \mid X)].$$
Common approaches for deriving ITRs with univariate right-censored outcomes include weighting and A-learning methods. 
However, it is difficult to generalize these existing methods to accommodate bivariate survival outcomes.

To overcome the above challenges, we model the treatment decision as a stochastic policy with a conditional distribution $\mathcal{D}(X):\mathcal{X} \rightarrow (0,1)^{K+1}$ given covariates $X$ \cite{fang2023fairness}.  
Define the individualized treatment distribution (ITD) of the optimal decision variable $D_0$ as $\mathcal{D}_{0}(X) = (d_{0}(0,X), \dots, d_{0}(K,X))^\top$, where $d_{0}(a,x)=P(D_0 = a \mid X=x)$. 
This distribution satisfies
\begin{align}\label{C2}
\mathcal{D}_{0}=&\argmin_{d} - E[\mathcal{S}(t_1^*, t_2^*, X; \eta^*)^\top d(X)],
\end{align} 
where 
\begin{align*}
&\mathcal{S}(t_1^*, t_2^*, x; \eta^*)  
= (S(t_1^*, t_2^*, x, 0, \eta_0), \dots, S(t_1^*, t_2^*, x, K, \eta_K))^\top,
\end{align*}
and $\eta^* = (\eta_0,\dots,\eta_K)$.
The optimal ITR therefore can be expressed as $D_0(x):= \argmin_{a \leq K} -d_0(a, x)$. 
Treating the optimal ITR as a conditional distribution of an optimal decision variable given covariates overcomes several difficulties of computational aspects to nonconvex optimization problems from classical settings \citep{fang2023fairness}.  
The proposed approach improves the establishment of ITR by accommodating uncertainty and heterogeneity in treatment effects.

\section{Estimation procedure}  \label{Sc:estimation}
In this section, we present a multi-stage estimation procedure for the decision-making tasks.
We briefly outline significant advantages of the proposed method as below:
\begin{itemize}
\item \textbf{Reliable estimation: } 
We propose an APP estimation approach for estimating the conditional potential joint survival function under the potential outcome framework with  right censoring, which directly improves the reliability of optimal decision rule learning though combining modern machine learning methods. 

\item \textbf{Flexibility: } 
We can employ any advanced black-box machine learning (ML) model to generate predictions. 
Even if this ML model is biased, selecting concrete weights in the estimating equation guarantees that the target estimators of parameters remain consistent.

\item \textbf{Computational efficiency: }
Once the survival parameters are estimated, the policy learning phase benefits from modern deep learning capabilities, such as stochastic policy and differentiability. 
In particular, different from traditional decision rules based on indicator function, which are non-differentiable, the proposed method uses DNNs with the softmax transformation to generate a probabilistic treatment assignment. 
This renders the objective function smooth and differentiable, enabling the use of efficient gradient-based optimization algorithms to handle complex estimation tasks.
\end{itemize}

\subsection{Adaptive prediction-powered estimation of joint survival functions} \label{Sc4.1}

In this study, we apply machine learning methods and prediction-powered inference method, and propose the APP estimation method for $\beta_{ja}$, $j=1,2$, $a\in\mathcal{A}$.

Beginning with a classical estimation strategy for $(\beta_{ja},\gamma_ja)$ in \eqref{AFT}, the standard IPCW least squares criterion for right-censored data \cite{li2003empirical}, for $j=1,2$ and $a\in\mathcal{A}$, we have \begin{align}\label{naive}
&(\hat{\beta}_{ja},\hat{\gamma}_{ja})  
=  \argmin_{(\beta,\gamma)} \sum_{i=1}^{n_a}\frac{ \Delta_i}{ \hat{G}_{ja}(Y_{ji},X_i)}  \frac{(\log Y_{ji} - \beta^{\top} X_i)^2}{\gamma^2},
\end{align} 
where $\hat{G}_{ja}$ is an estimator of $G_{ja}(t,x) := P(\Delta_j = 1 \mid T_j=t,A=a,X=x)$ \cite{cui2023estimating}. 

The above formulation can be interpreted through a missing data perspective: there exists an informative missing mechanism $P(\Delta = 1 \mid T, X)$ managing the observation of the label, which is determined by the conditional distribution of $C$.
Consequently, information regarding the censoring time $C$ helps the identification of this latent mechanism, analogous to the role of the propensity score in causal inference for observational studies \cite{rosenbaum1983central}.
On the other hand, in randomized controlled trials, the estimator of $\beta_{ja}$ captures the causality between covariates and potential outcomes.
Consequently, imputations based on the desirable $\hat{\beta}_{ja}$ and observed explanatory variables can contribute to decision learning tasks. 

Moreover, we can develop more efficient estimators through prescriptive of the prediction-powered estimations \citep{angelopoulos2023prediction, zrnic2024cross}.
For $j=1,2$ and $a\in\mathcal{A}$, letting the score function be
\begin{align*}
&\psi_{ja}^* (\beta,f,G,\kappa;O,c)\\
=  & \frac{X\Delta}{G(Y_j,X)}( \log Y_j - X^{\top} \beta )  +  c_1 \frac{I(A=a)}{\kappa} \{X(f(X)-X^{\top} \beta)\}  \nonumber\\ 
&+  c_2 \frac{I(A\neq a) }{1-\kappa} \{X(f(X)-X^{\top} \beta)\}  ,
\end{align*}
we propose an estimator of $\beta_{ja}$, denoted by $\hat{\beta}^{APP}_{ja}$, which is obtained by solving the following estimating equation with respect to $\beta$
\begin{align} \label{PLS2} 
& \mathbb{P}_n \psi_{ja}^* (\beta,\hat{f}^{(a)}_{jn},\hat{G}_{ja},\hat{\kappa}_n;\mathcal{O},c)=0.
\end{align}   
Here, $\hat{\kappa}_n $ is an estimator of $\kappa_a:=P(A=a)$. 
$c=(c_1,c_2) $ is the user-specified weight, and $\hat{f}_{jn}^{(a)}:=\hat{E}[\log T_{j}(a) \mid X=x ]$ is developed by machine learning methods, which approximates the true function $f_{j0}^{(a)}(x):=E[\log T_{j}(a) \mid X=x]$.
Further, $\hat{\kappa}_n$, the estimator of the probability $P(A=a)$, can be computed as the empirical frequency of the event observed in $\mathcal{O}$.  

\begin{remark}[Estimation methods for $G_{ja}$]
Various methods exist for estimating the conditional survival probability of the censoring time given covariates, including logistic regression \cite{matsouaka2020regression}, Cox Proportional Hazards (CPH) modeling \cite{cox1975partial}, local Kaplan–Meier estimation \cite{tang2020penalized}, and machine learning approaches \cite{JMLR}.
\end{remark}

\begin{remark}[Estimation methods for $f_0$]
One way to obtain $\hat{f}_{jn}^{(a)}$ is to solve following $M$-estimation problem:
\begin{align}\label{PN}
\hat{f}_{jn}^{(a)} = \argmin_{f\in\mathcal{F}} \frac{1}{n_a}\sum_{i=1}^{n_a} \biggl\{ \frac{\Delta_{ji}}{\hat{G}_{ja}(Y_{ji},X_i)} (\log Y_{ji}-f(X_i))^2 \biggr\},
\end{align}
where $\mathcal{F}$ is the family of approximation tools, such as neural networks, random forest and others.    
Here we assume that the estimator $\hat{f}_{jn}^{(a)}$ converges to $f_{j0}^{(a)} $.
This is not a strong assumption as there exist several theoretical guarantees about the error bounds of machine learning methods in nonparametric estimation tasks, and these theories demonstrate that the key hyperparameters can be adjusted suitably so that the required convergence rate is achievable. 
\end{remark}

\begin{remark}[The role of weight $c$]
Straightforward calculation implies that when $c=(1,-1)$ or $c=(-1,1)$, $E[\psi_{ja}^* (\beta_{ja},f,G_0,\kappa_a;\mathcal{O},c)] = 0$ holds for any measurable function $f:\mathcal{X}\rightarrow \mathbb{R}$.
In this case, the estimating equation $\psi_{ja}^*$ involves a debias procedure with respect to $f$, as similarly presented by \cite{angelopoulos2023prediction, zrnic2024cross} in the semi-supervised settings, which presents the most conservative estimation strategy for $\beta_{ja}$.

On the other hand, if the estimator $\hat{f}^{(a)}_{jn}$ is consistent, we can develop more efficient estimator $\hat{\beta}_{ja}^{APP}$ through selecting a weight.
In the simulation study, we provide some valid selections of $c$ and evaluate their finite-sample performance.
\end{remark}

After developing $(\hat{\beta}_{10},\dots,\hat{\beta}_{1K})$ and $(\hat{\beta}_{20},\dots,\hat{\beta}_{2K})$, supposing that censoring times $C_1$ and $C_2$ are conditionally independent given $X$, we can fit copula models through  
\begin{align}\label{EST1}
\hat{\theta}_a=\argmin_{\theta}- \sum_{i=1}^{n}&\frac{ \Delta_{1i}\Delta_{2i}I(A_i=a)}{\hat{G}_{1a}(Y_{1i},X_i)\hat{G}_{2a}(Y_{2i},X_i)} \nonumber\\
&\cdot\log \bigg[\frac{\partial^2}{\partial t_{1} \partial t_{2}} S(Y_{1i},Y_{2i},X_i, a,(\hat{\beta}_{1a}, \hat{\beta}_{2a}, \hat{\gamma}_{1a},\hat{\gamma}_{2a}, \theta))\bigg],
\end{align}
where $S(t_1,t_2,x, a,(\hat{\beta}_{1a}, \hat{\beta}_{2a},\hat{\gamma}_{1a},\hat{\gamma}_{2a}, \theta)) = L(S(t_1,x,a; (\hat{\beta}_{1a},\hat{\gamma}_{1a})), S(t_2,x,a; \hat{\beta}_{2a},\hat{\gamma}_{2a}); \theta)$.

\subsection{Optimal decision distribution learning} \label{Sc4.2}
In this section, we present a procedure for estimating the optimal decision distributions.

First, to validly model the ITD, we introduce the softmax function $\sigma_S$. 
For $y=(y_0,\dots,y_K)^{\top} \in\mathbb{R}^{K+1}$,
\begin{align*}
\sigma_S(y) := (  [\sigma_S(y)]_0,\dots,[\sigma_S(y)]_K)^{\top} :\mathbb{R}^{K+1}\rightarrow [0,1]^{K+1},
\end{align*}
where $[\sigma_S(y)]_a =  e^{y_a}/\sum_{i=0}^{K}e^{y_i}$, $a\in\mathcal{A}$ such that $\sum_{i=0}^{K} [\sigma_S(y)]_i = 1$.   

The softmax function is widely used in large language model training, mapping character encoding information to probability distribution introduces uncertainty into text generation decisions, thereby improving the generalization performance of the model.

Figure \ref{EP} displays the decision-making process using a combination of neural networks and the Softmax function:
\begin{itemize}
\item[(i)] Deep neural networks are applied to model the latent relationships between covariates and scores, values of which represent the tendency of treatment assignment.

\item[(ii)] Softmax function is integrated to the outputs, so that results can be interpreted as the conditional distribution of the optimal decision variable.

\item[(iii)]  Eventually, the predicted values are used to make decisions: the treatment plan corresponding to the maximum value is recommended and adopted (e.g. =0 shown in the figure).
\end{itemize}

\begin{figure} 
\begin{center}
\includegraphics[width=0.8\linewidth]{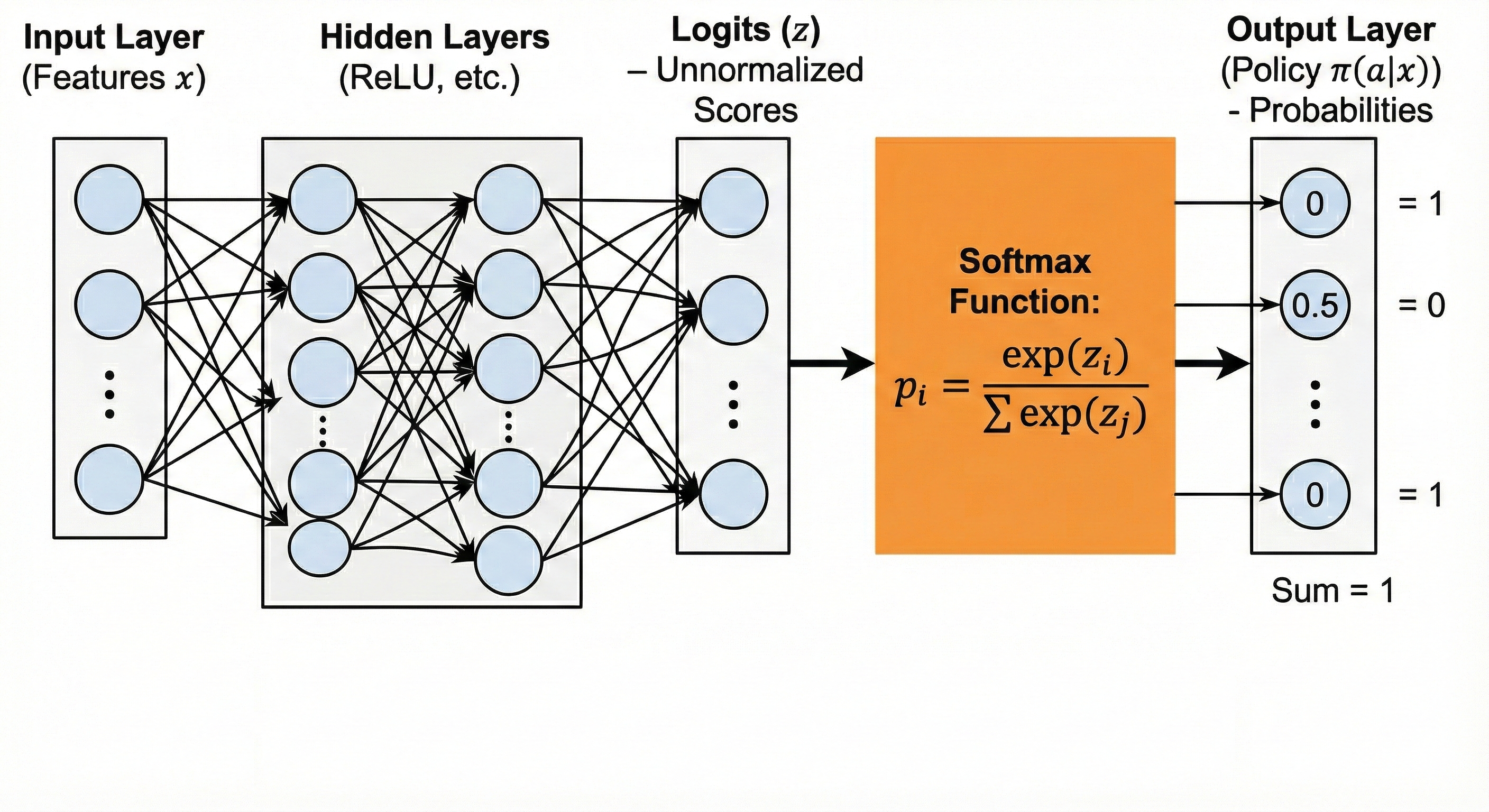}
\caption{Decision making by combining deep neural networks with the softmax function for multiple treatments.}
\label{EP}
\end{center}
\end{figure}  

In this study, we apply $(K+1)$-outcome DNNs with ReLU activation functions to estimate $\mathcal{D}_0$ for establishing the optimal decision distribution. 
For a neural network with depth $\mathcal{D}$ and width $\mathcal{W}$, let $(A_1,\ldots,A_{\mathcal{D}})$ and $(V_1,\ldots,V_{\mathcal{D}})$ denote the weight matrices and bias vectors, respectively.  
Then the DNNs with input $x $ and ReLU activation function $\sigma_j(x):=\max\{x,0\}$, $j=1,\ldots,\mathcal{D}-1$ can be expressed as $H(x):= A_{\mathcal{D}}\sigma_{\mathcal{D}-1}\big(\ldots\sigma_1\big(A_1x+V_1\big)\ldots\big)+V_{\mathcal{D} } .$
Let $\|x\|_{\infty}$ denote the maximum norm of vector $x$.
We assume that $A_j\in\mathcal{M}_j:=\{A\in \mathbb R^{d_j\times d_{j-1}}:|A_{kl}|\leq1,1\leq k\leq d_j,1\leq l\leq d_{j-1}\}$ and $V_j\in\mathcal{V}_j:=\{V\in\mathbb R^{d_j}:\|V\|_{\infty}\leq 1\} ,$ for each $ j=1,\ldots,\mathcal{D}.$
We denote the family of the uniformly bounded DNNs with $d_0=p,\ d_{\mathcal{D}}=K+1$ and sufficiently large constant $B$ as 
$\mathcal{F}(B):=\{ H(x): \mathbb{R}^{d_0}\rightarrow \mathbb{R}^{d_{\mathcal{D}+1}}, A_j\in\mathcal{M}_j,\ V_j\in\mathcal{V}_j,\ \forall j=1,\ldots,\mathcal{D}+1, \sup_{x \in \mathcal{X}}|H(x)| \leq B \}$.  

Consequently, combing the softmax function, we have following estimator,  
\begin{equation}\label{EST_D2}
\hat{h}_n = \argmin_{h\in \mathcal{F}(B)} -\sum_{i=1}^{n}[\mathcal{ S}(t_1^*,t_2^*, X_i;\hat{\eta})^\top \sigma_S(h(X_i))],
\end{equation}   
where $\hat{\eta}$ is an estimator of $\eta^*$.
Such an estimator $\hat{h}_n$ equivalently maximizes the empirical value based on the estimated survival function at $(t_1^*,t_2^*)$ and completes a multi-label classification task related to identify ranks of potential joint conditional survival probabilities. 

The estimated decision $\hat{D}: \mathcal{X} \rightarrow \mathcal{A}$ for an individual with $X=x$ is constructed as
\begin{align}\label{Dec}
\hat{D}(x) = -\argmin_{a\leq k} [\hat{h}(x)]_a,
\end{align}
where $[\hat{h}(x)]_j$ is the $j$th component of vector $\hat{h}_n(x)$. 
Similar idea is applied by \cite{lou2018optimal} for identifying the optimal treatment assignment with multiple treatments.
We provide more explanatory function $h_a$, $a\in\mathcal{A}$, which contributes to characterize the optimal ITD.
We also separate decision specification from optimal policy distribution training, avoiding direct optimization of indicator functions or variants to optimize computational procedures.

\subsection{Algorithm}
Combining the estimation procedures shown in Subsections \ref{Sc4.1} and \ref{Sc4.2}, we provide Algorithm \ref{alg:app_itr}.

\begin{algorithm}[H]
\caption{Deep Optimal ITR with Adaptive Prediction-Powered (APP) Estimation}
\label{alg:app_itr}
\begin{algorithmic}[1]
\STATE {\bfseries Input:} Observed data $\mathcal{O}=\{(Y_{1i}, Y_{2i}, \Delta_{1i}, \Delta_{2i}, X_i, A_i)\}_{i=1}^n$, target time $(t_1^*, t_2^*)$, weights $c=(c_1,c_2)$.
\STATE {\bfseries Output:} Optimal Decision Rule $\hat{D}(x)$.

\STATE \textbf{// Phase 1: Nuisance and Marginal Parameter Estimation (Section \ref{Sc4.1})}
\FOR{each treatment $a \in \{0, \dots, K\}$}
\FOR{each outcome $j \in \{1, 2\}$}
\STATE Develop $\hat{G}_{ja}$.
\STATE Train stable predictor $\hat{f}_{jn}^{(a)}$.
\STATE Calculate $\hat{\kappa}_n$.
\STATE Compute the APP estimator $\hat{\beta}_{ja}^{APP}$ by solving \eqref{PLS2} with the user-specified weight $c$:
\STATE $\quad \mathbb{P}_n \psi_{ja}^* (\beta,\hat{f}^{(a)}_{jn},\hat{G}_{ja},\hat{\kappa}_n;\mathcal{O},c)=0$ for $\beta$.
\STATE Calculate estimators of $\gamma_{ja}$ via \eqref{naive}. 
\ENDFOR
\STATE \textbf{// Phase 2: Joint Dependence Estimation (Section \ref{Sc4.1})}
\STATE Estimate Copula parameter $\hat{\theta}_a$ by cross validation and maximizing pseudo-likelihood \eqref{EST1}.
\ENDFOR

\STATE \textbf{// Phase 3: Optimal Decision Distribution Learning (Section \ref{Sc4.2})}
\STATE Initialize Deep Neural Network $h: \mathcal{X} \rightarrow \mathbb{R}^{K+1}$.
\STATE Define Softmax policy: $\pi(a|x) = [\sigma_S(h(x))]_a$.
\STATE Construct the estimator of joint survival function: $\mathcal{S}(t_1^*, t_2^*, x;\hat{\eta})= (S(t_1^*, t_2^*, x, 0; \hat{\eta}_0), \dots,S(t_1^*, t_2^*, x, K; \hat{\eta}_K))$.
\WHILE{not converged}
\STATE Sample batch $\mathcal{B}$ from data.
\STATE Update $h$ by maximizing empirical value \eqref{EST_D2}:
\STATE $ \sum_{i \in \mathcal{B}} \{ \sum_{a=0}^K S(t_1^*, t_2^*, X_i, a;\hat{\eta}_a)  [\sigma_S(h(X_i))]_a \}$.
\ENDWHILE

\STATE \textbf{Return} Optimal ITR: $\hat{D}(x) = \arg\max_{a} [\hat{h}(x)]_a$.
\end{algorithmic}
\end{algorithm}

\section{Theoretical properties}
\label{sec:theory}

In this section, we establish the asymptotic properties of the proposed estimator and the learned policy. We focus on two main results: 
(i). the consistency of the proposed APP estimator for the marginal and copula parameters; and 
(ii). the risk consistency of the estimated optimal treatment rule.

\subsection{Assumptions}

We impose the following regularity conditions, which are standard in survival analysis and empirical risk minimization.

\begin{assumption}[Regularity of the Survival Model]
\label{ass:regularity}
\begin{itemize}
\item[(i)] The space that contains true parameter $\eta_{a} = (\beta_{1a}, \beta_{2a}, \gamma_{1a},\gamma_{2a}, \theta_{a})$ is compact. 
\item[(ii)] The covariate vector $X\in\mathcal{X}$ is bounded. The matrix $E[XX^\top]$ is non-singular.
\item[(iii)] The censoring time $(C_1,C_2)$ is independent of $(T_1, T_2)$ conditional on $X$ and $A$. 
The probability of being uncensored is uniformly bounded away from zero: $P(\Delta_j =1 \mid X,A,T_j) > \delta $ with a positive constant $\delta$ for $j=1,2$.
\end{itemize}
\end{assumption}

\begin{assumption}[Stability of the ML Predictor]
\label{ass:predictor}
The machine learning predictor $\hat{f}$ belongs to a function class $\mathcal{F}$ with finite complexity, such as finite VC-dimension or Rademacher complexity. 
Assume that $\hat{f}$ converges in probability to a limiting function $f^*$, which does not necessarily equal the true function.
\end{assumption}

\begin{assumption}[Identifiability of Copula]
\label{ass:copula}
The Copula density function $c_\theta(u, v) = \partial^2 L_\theta(u,v)/\partial u \partial v$ is twice continuously differentiable with respect to $\theta$. The Fisher information matrix for the pseudo-likelihood is positive definite.
\end{assumption}

Here Assumption \ref{ass:regularity} is standard in the survival analysis \cite{marra2020copula,matsouaka2020regression,tang2020penalized}.
Assumption \ref{ass:predictor} is proposed by \cite{zrnic2024cross}, which ensures prediction-powered inference.
Assumption \ref{ass:copula} is sufficient to guarantee the identifiability of Copula model, which can be satisfied by several copula link functions \cite{marra2020copula}.

\subsection{Consistency of the APP Estimator}

First, we show that the APP estimator for the marginal parameters $\beta_{ja}$ is consistent despite potentially biased machine learning predictions.

\begin{theorem}[Consistency of Estimators]
\label{thm:consistency_beta}
Suppose that Assumption \ref{ass:regularity} holds.

\begin{enumerate}
\item[(i)] Under Assumption \ref{ass:predictor}, suppose that $\hat{G}_{ja}$ is a consistent estimator of $G_{ja}$. Then, for any treatment $a \in \mathcal{A}$ and outcome $j \in \{1, 2\}$, the APP estimator $\hat{\beta}_{ja}^{APP}$ with weights $c=(1,-1)$ or $c=(-1,1)$ satisfies:
$$ \hat{\beta}_{ja}^{APP} = \beta_{ja} + o_p(1). $$

\item[(ii)] Suppose that $\hat{G}_{ja}$ and $\hat{f}_{jn}^{(a)}$ are consistent estimators of $G_{ja}$ and $f_{j0}^{(a)}$, respectively. Then, for any treatment $a \in \mathcal{A}$, outcome $j \in \{1, 2\}$, and any weight vector $c$, the APP estimator satisfies:
$$ \hat{\beta}_{ja}^{APP} = \beta_{ja} + o_p(1). $$
\end{enumerate}
\end{theorem}

Theorem \ref{thm:consistency_beta} implies that there exists a trade-off between robustness and efficiency.
In particular, when $c=(1,-1)$ or $c=(-1,1)$, the estimation procedure parallels prediction-powered estimation, meaning that biases induced by $\hat{f}^{(a)}_{jn}$ can be corrected. 
However, the variance information captured by $\hat{f}^{(a)}_{jn}$ is ignored during this debiasing process, resulting in a loss of statistical efficiency.
Therefore, the weights can be adjusted based on the user's level of trust in the machine learning method on the concrete dataset.
For instance,  when the prediction performance of estimator $\hat{f}^{(a)}_{jn}$ on the test set falls short of expectations, we can choose more robust weights to reduce potential risks.

Then, the consistency of estimator $\hat{\gamma}_{ja}$ obtained by \eqref{naive} follows from the similar proofs of Theorem \ref{thm:consistency_beta} and thus omitted. 

\begin{corollary}[Consistency of Copula Parameter]\label{co1}
Given consistent estimators $\hat{\beta}^{APP}_{ja}$, $\hat{\gamma}_{ja}$ and Assumption \ref{ass:copula}, the two-stage estimator $\hat{\theta}_a$ obtained by \eqref{EST1} converges to the true value $\theta_a$ in probability, i.e., $\hat{\theta}_a = \theta_a + o_p(1)$ for any $a\in\mathcal{A}$.
\end{corollary}

Please refer to \cite{andersen2005two} for the proof of Corollary \ref{co1}.

\subsection{Value consistency of the learned policy}

Let $V(d) = E[\mathcal{S}(t_1^*, t_2^*, X;\eta^*)^\top d(X)]$ be the expected value under policy $d$.  

\begin{theorem}[Value Consistency]
Under Assumptions \ref{ass:regularity}-\ref{ass:copula}, for $\hat{h}_n$ obtained by \eqref{EST_D2}, when estimators $\hat{\beta}_{ja}^{APP}$ and $\hat{\theta}_a$ are consistent, we have 
\label{thm:risk_consistency}  
$$ V(\mathcal{D}_0) = V(\sigma_{S}(\hat{h}_n))  + o_p(1). $$
\end{theorem}

Theorem \ref{thm:risk_consistency} describes that the estimator of ITD can perform as well as the optimal ITD asymptotically, when our goal is to achieve the  optimal value.

\section{Conclusion and limitations}
\textbf{Conclusion} In this work, we propose a robust approach for the construction of optimal individualized treatment rules under bivariate survival models.
The proposed method integrates prediction-powered inference and deep learning approach to provide a robust way to leverage auxiliary information from black-box prediction machines.
The consistency of the proposed estimators is established.
This theoretical guarantee allows us to safely employ deep neural networks for nonparametric estimation tasks, thereby mitigating the risks of model misspecification and data heterogeneity. 
Both the theoretical and experimental results show the significant potential of the proposed approach in addressing complex decision making challenges.

\textbf{Limitation} This study focuses on the parametric modeling of bivariate time-to-event data.
Extending the proposed methods to nonparametric estimation and inference represents a logical direction to support optimal treatment rule learning. We leave these methodological extensions for future research.

\bibliography{ref}
\bibliographystyle{chicago} 

\appendix
\onecolumn
\section{Appendix}
\subsection{Simulation studies}\label{sec:real_data}
\textbf{Data Generating Process}

We implement simulation studies to evaluate the finite-sample performance of the proposed method in the presence of treatment heterogeneity. 
We consider a scenario with $K=2$ and a $2$-dimensional covariate vector.
For each subject $i=1, \dots, n$, the covariates $ {X}_i = (X_{1i}, X_{2i})^\top$ are generated independently from a uniform distribution:
$X_{pi} \sim \text{Uniform}(-2.8, 2.8), \quad p=1, 2$.
The treatment assignment $A_i$ is randomized uniformly among the three available treatments.
The failure times for two potential outcomes, $T_{1i}(a)$ and $T_{2i}(a)$, are generated as
$$\log T_{ji} (a) =  {X}_i^\top  {\beta}_{ja} + X_{2i}^2 +\epsilon_{jai},$$ 
$j=1,2$, where $ {\beta}_{ja}$ represents the regression coefficients for outcome $j$ under treatment $a$, and independent
$\epsilon_{1ai}, \epsilon_{2ai} \sim \mathcal{N}(0, 1)$.
The censoring times $C_{1}$ and $C_{2}$ are drawn independently from an uniform distribution $\mathcal{U}(-\tau, 2\tau)$ to achieve a $50\%$ censoring rate. 
The observed data consists of $Y_{ji} = \exp(\min(\log T_{ji}, C_{ji}))$, where $\log T_{ji} = I(A=1) \log T_{ji}(1) + I(A=0) \log T_{ji}(0)$, and the event indicator $\Delta_{ji} = I(\log T_{ji} \le C_{ji})$.
In this case, although the working model \eqref{AFT} is misspecified, target parameters ${\beta}_{ja}$ can be estimated consistently via least squares.

The regression coefficients are designed to create distinct decision boundaries across the covariate space:
\begin{itemize}
\item $ {\beta}_{10} = (1.5, 1.0)^\top$ and $ {\beta}_{20} = (1.0, 1.5)^\top$.
\item $ {\beta}_{11}  = (-1.5, 1.0)^\top$ and $ {\beta}_{21}  = (-1.0, 1.5)^\top$.
\item $ {\beta}_{12}  = (0.0, -2.0)^\top$ and $ {\beta}_{22}  = (0.0, -2.0)^\top$.
\end{itemize} 
The true joint survival probability for treatment $a\in\{0,1,2\}$ at a fixed time point $(t_1, t_2)$ is 
\begin{align*}
S(t_1, t_2 ,x,a;\eta_a) 
&= ( S_1(t_1,x,a;(\beta_{1a},\gamma_{1a}))^{- \theta_a} \\
& + S_2(t_2,x,a;(\beta_{2a},\gamma_{2a}))^{- \theta_a} - 1 )^{-1/\theta_a},
\end{align*}  
where the marginal survival probabilities are derived from the AFT model.
The dependence parameters are set to $ {\theta} = (2.0, 2.5, 3.0)$ for treatments 0, 1, and 2, respectively.  

\textbf{Implementation Details and Evaluation Metrics}

We implement Algorithm \ref{alg:app_itr} using a training sample size of $n=200$ and evaluate the results over $R$-times independent Monte Carlo replications.
In particular, $G_{ja}$ is estimated through CPH modeling, and $f^{(a)}_{j0}$ is approximated by the standard nonparametric IPCW least square with the random forest.	
The target of the optimal ITR is to maximize the individual joint survival probability at the fixed time point $(t_{1}^*, t_{2}^*) = (1, 1)$.

We use the following optimal treatment identification accuracy (OTIA) as the primary evaluation metric. 
In particular, we generate an independent testing set ($n_{test}=1000$), on which the corresponding decisions $\hat{D}( {X})$ are predicted, and compared with the true optimal treatment rule, $D_0( {X})$:
$$\text{OTIA} = \frac{1}{n_{test}} \sum_{i=1}^{n_{test}} I(\hat{D}( {X}_i) = D_0( {X}_i)).$$
The true optimal individualized rule $D_0$ is derived from the ground truth survival functions: $D_0  = \arg\max_{a \in \mathcal{A}} S(t_1^*, t_2^* , x, a;\eta_a)$.

We compare the performance of the proposed method under three distinct settings with different weights $c=(c_1, c_2)$:
\begin{enumerate}
\item \textbf{Baseline (Standard IPCW):} $ {c}=(0,0)$ gives the standard parametric IPCW least-squares estimator, serving as the benchmark without leveraging any auxiliary prediction information.
\item \textbf{Prediction-Base (PB) estimation:} $ {c}=(1,1)$ is designed to leverage the information from black-box model, and prediction results are trusted.
\item \textbf{Robust PP estimation:} $ {c}=(-1,1)$ corresponds to the formulation of Prediction-Powered (PP) estimation under right-censored situations.  
\end{enumerate}
The code is available at \url{https://anonymous.4open.science/r/BITR-DFDE/}.

\textbf{Results and Discussion}

The average OTIA across $R=100$ replications for the estimated ITRs under the three weight configurations are: 
Baseline ($ {c}=(0,0)$): $  0.9254 $;
PB ($ {c}=(1,1)$): $ \textbf{0.9578} $;
PP ($ {c}=(-1,1)$): $ 0.9382 $.

The results demonstrate the following key findings:
\begin{enumerate}
\item \textbf{Outstanding Performance of $c=(1,1)$:}
The PB estimator yields the best OTIA.
This result shows that when high-quality predictors are available, the PB estimation approach effectively leverages this auxiliary information to match the high decision-making accuracy of methods based on correctly specified models.

\item \textbf{Decision Boundary Discovery:} Figures \ref{fig2}-\ref{fig4} display the estimated decision boundary in \textit{a single replication} on the testing set. 
In particular, the left panels of each figure present the true optimal treatment assignments for every testing subject, which are displayed as scatter plots in different colors.
Meanwhile, the right panels contains two elements: (i) predicted optimal individualized treatment assignments; and (ii) empirical decision boundaries (colored background) that distinguish diverse treatment assignments via individual features.
The deep neural network-based ITR successfully captures the relationship between covariates and the optimal treatment rank, closely mimicking the ground truth partition under Scenarios where $c=(1,1)$.

\item \textbf{Robustness vs. Efficiency Trade-off:} 
The PP estimation exhibits the medium accuracy.
While it guarantees consistency under predictor misspecification, in this setting where the predict machine $\hat{f}$ is highly accurate, its inherent robustness mechanism introduces additional estimation variability or bias relative to PP. 
This highlights the critical trade-off: when the predictor is reliable, the efficiency-focused method is empirically superior for ITR learning.

\end{enumerate}

\begin{figure}[H]
\begin{center}
\includegraphics[width=0.8\linewidth]{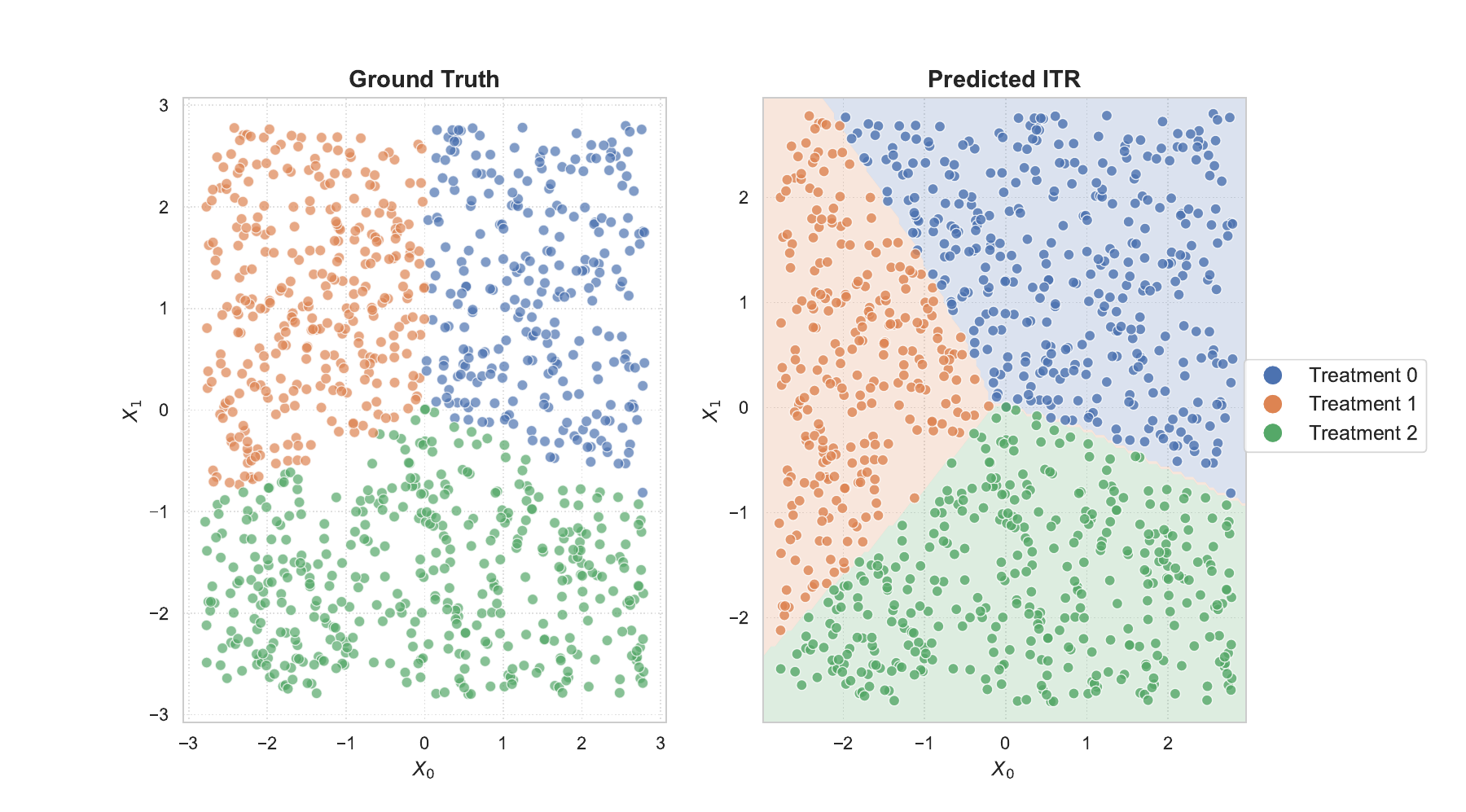}
\caption{Visualization of ground truth and predicted Individualized Treatment Rules with $c=(0,0)$.}
\label{fig2}
\end{center}
\end{figure} 

\begin{figure}[H]
\begin{center}
\includegraphics[width=0.8\linewidth]{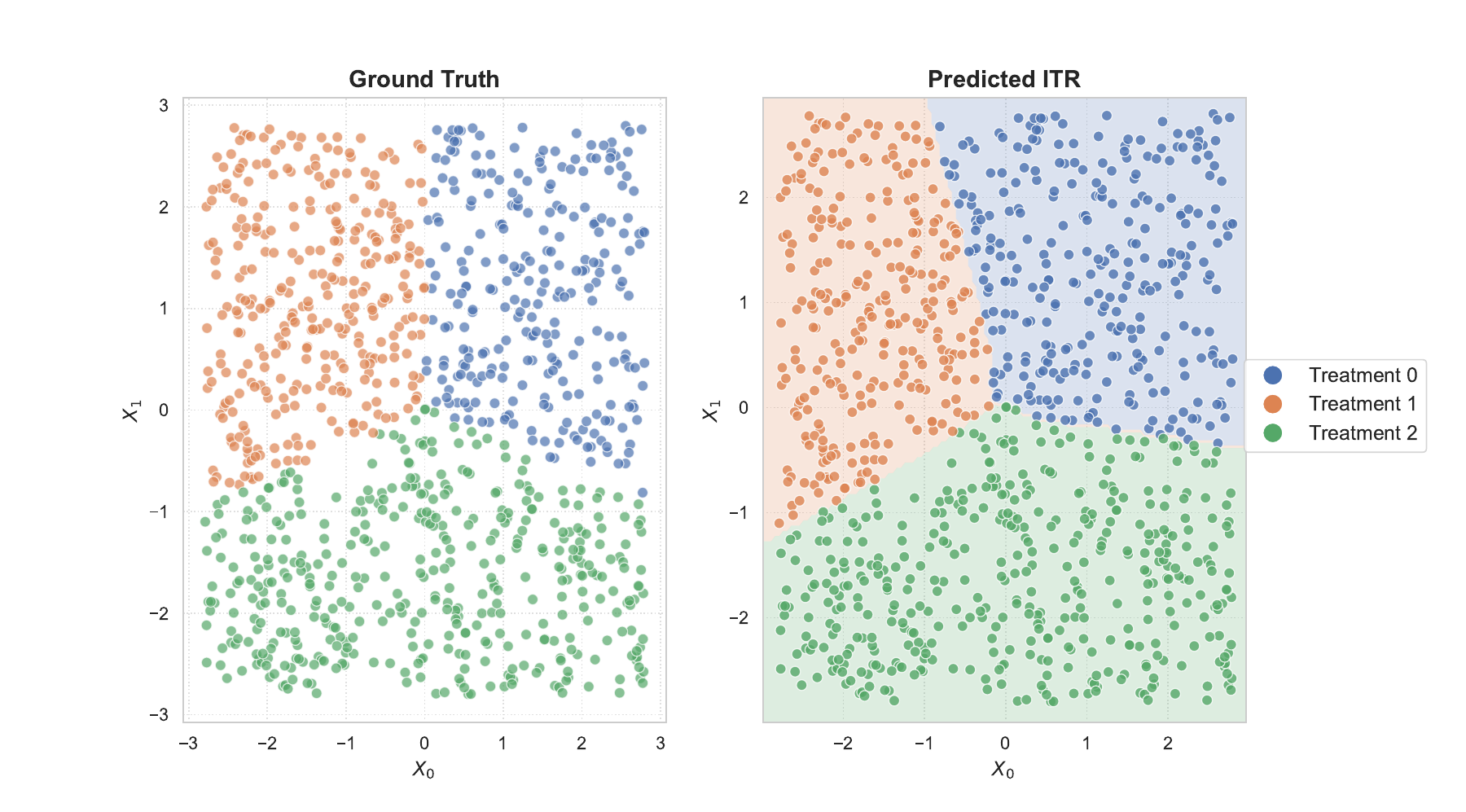}
\caption{$c=(-1,1)$.}
\label{fig3}
\end{center}
\end{figure} 

\begin{figure}[H] 
\begin{center}
\includegraphics[width=0.8\linewidth]{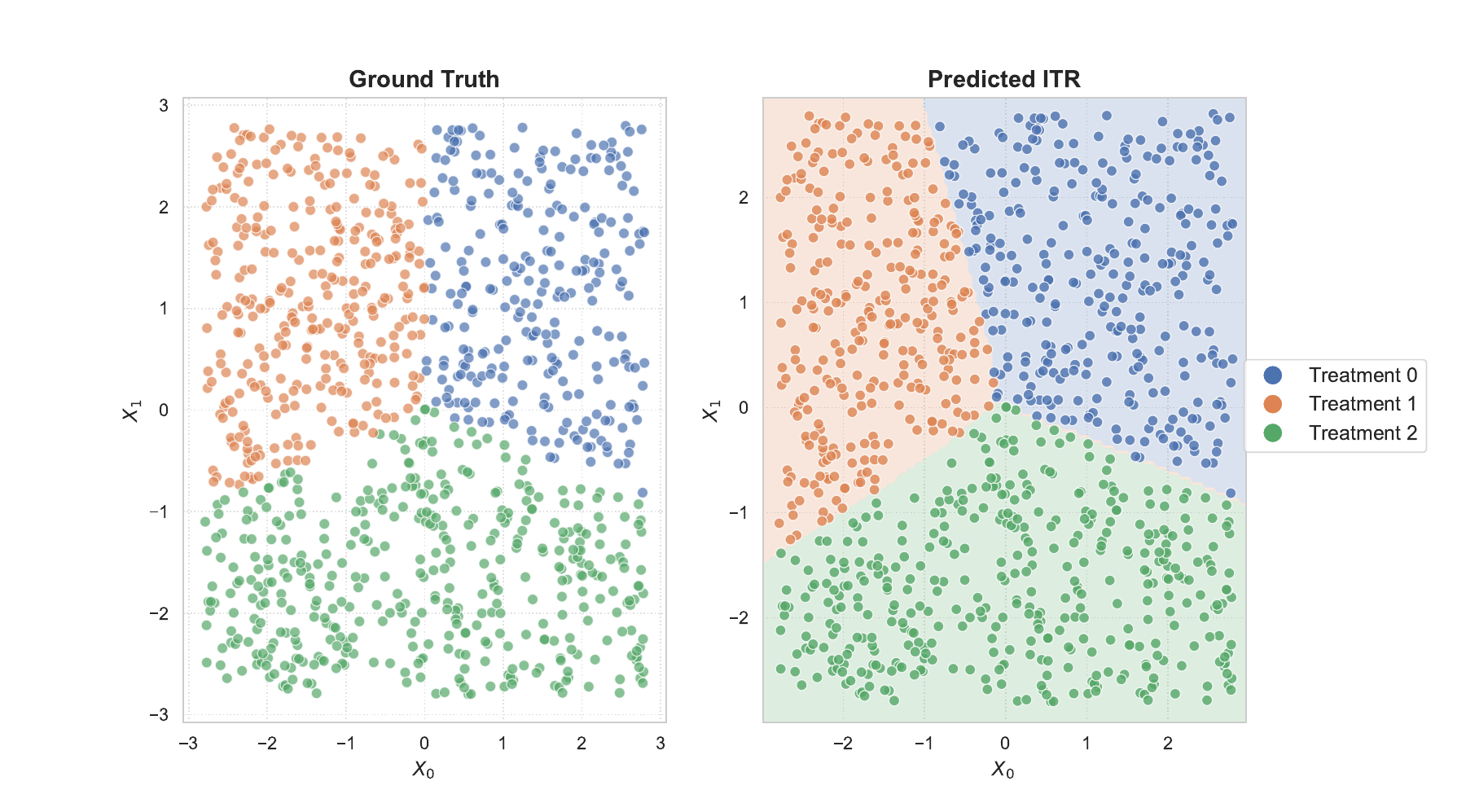}
\caption{$c=(1,1)$.}
\label{fig4}
\end{center}
\end{figure}

\subsection{An additional simulation study}
We consider additional situations where potential failure times, $T_{1i}(a)$ and $T_{2i}(a)$, are generated as
$$\log T_{ji} (a) =  {X}_i^\top  {\beta}_{ja} + \epsilon_{ji},$$
$a=0,1,2$, $j=1,2$ with
\begin{itemize}
\item[] Case 1. $\epsilon_{1i} | (X_{1i}, X_{2i})\sim \mathcal{N}(0, X_{1i}^2)$ and $\epsilon_{2i} | (X_{1i}, X_{2i})\sim \mathcal{N}(0, X_{2i}^2)$.

\item[] Case 2. $\epsilon_{1ai}, \epsilon_{2ai} \sim \mathcal{N}(0, 1)$.
\end{itemize}
Other settings are same to the data generating process shown in Section \ref{sec:real_data}.
Under the above two settings, we apply the working model \eqref{AFT}. 
Then, Algorithm \ref{alg:app_itr} is implemented with a training sample size of $n=200$ and $100$ replications.

\textbf{Case 1.}
The average OTIA across $R=100$ replications for the estimated ITRs under the three weight configurations are: 
Baseline ($ {c}=(0,0)$): $0.9839$;
PB ($ {c}=(1,1)$): $0.9831$;
PP ($ {c}=(-1,1)$): $0.9822$.

Briefly, all of these three methods successfully capture the latent relationships between covariates.
Then, the desired estimators contribute to well training of Softmax deep neural networks, leading to establishment of the optimal decision empirically.
Figures \ref{fig8}-\ref{fig10} display the estimated decision boundary in a single replication on the testing set.  

\textbf{Case 2.}
Table \ref{tb2} summarizes the estimating results for coefficients under different weight selections.
The simulation results across different weight settings indicate that the weight $(0,0)$ provides the most robust estimation for the coefficients $\beta_{ja}$. 
Specifically, such a baseline achieves the minimum absolute average bias and the lowest sample standard deviation (SSD) across all treatment levels ($a=0, 1, 2$) and outcomes ($j=1, 2$), presenting superior accuracy and efficiency. 

In contrast, PB ( $c=(1,1)$ ) and PP ($c=(-1,1)$) methods exhibit significantly higher biases and increased variability. 
The performance degradation is most pronounced under the PP approach, which yields the largest SSDs in most scenarios, likely reflecting a higher sensitivity to the heteroscedastic error structure defined in the data generation process. 
Such a conclusion is similar to the results of comparing prediction-powered estimation and least square method in semi-supervised linear regression problems as shown by \cite{zrnic2024cross}.

The average OTIA across $100$ replications for the estimated ITRs under the three weight configurations are: 
Baseline ($ {c}=(0,0)$): $0.9750$;
PB ($ {c}=(1,1)$): $0.9736$;
PP ($ {c}=(-1,1)$): $0.9716$.
Figures \ref{fig5}-\ref{fig7} display the estimated decision boundary in a single replication on the testing set. 
It is clear to see that only points located on the decision boundary are at risk of being misclassified based on all three methods.
Combining the metrics, OTIA and results shown in figures, we conclude that although the estimated policy is relatively rough in a single experiment, deep neural networks are able to perform well in average.

\begin{table}[htbp]
\centering
\small
\caption{The absolute average bias (Bias) and sample standard deviation (SSD) for the estimates of coefficients $\beta_{ja}$ under different weight settings.}
\label{tb2}
\begin{tabular}{cccccccccc}
\hline
\multirow{2}{*}{Action} & \multirow{2}{*}{Outcome} & \multirow{2}{*}{$\beta$} & \multicolumn{2}{c}{$(1,1)$} & \multicolumn{2}{c}{$(0,0)$} & \multicolumn{2}{c}{$(-1,1)$} \\
\cmidrule(lr){4-5} \cmidrule(lr){6-7} \cmidrule(lr){8-9}
& & & Bias & SSD & Bias & SSD & Bias & SSD \\
\hline
\multirow{4}{*}{0} & \multirow{2}{*}{1} & $\beta_{11}$ & 0.0845 & 0.2530 & 0.0190 & 0.2396 & 0.0548 & 0.3223 \\
&                    & $\beta_{12}$ & 0.1062 & 0.1890 & 0.0010 & 0.1605 & 0.0996 & 0.2173 \\  
& \multirow{2}{*}{2} & $\beta_{21}$ & 0.1253 & 0.1848 & 0.0044 & 0.1677 & 0.1000 & 0.2524 \\
&                    & $\beta_{22}$ & 0.1048 & 0.2755 & 0.0035 & 0.2305 & 0.0871 & 0.2952 \\

\multirow{4}{*}{1} & \multirow{2}{*}{1} & $\beta_{11}$ & 0.0353 & 0.2773 & 0.0204 & 0.2291 & 0.0377 & 0.3026 \\
&                    & $\beta_{12}$ & 0.1450 & 0.2036 & 0.0426 & 0.1793 & 0.1294 & 0.2361 \\  
& \multirow{2}{*}{2} & $\beta_{21}$ & 0.1655 & 0.2168 & 0.0218 & 0.1991 & 0.1471 & 0.2407 \\
&                    & $\beta_{22}$ & 0.0933 & 0.2629 & 0.0051 & 0.2323 & 0.0569 & 0.2769 \\

\multirow{4}{*}{2} & \multirow{2}{*}{1} & $\beta_{11}$ & 0.0017 & 0.1835 & 0.0212 & 0.2053 & 0.0081 & 0.2240 \\
&                    & $\beta_{12}$ & 0.0247 & 0.1592 & 0.0369 & 0.1721 & 0.0467 & 0.1938 \\  
& \multirow{2}{*}{2} & $\beta_{21}$ & 0.0468 & 0.1864 & 0.0058 & 0.1889 & 0.0290 & 0.2192 \\
&                    & $\beta_{22}$ & 0.0691 & 0.2774 & 0.0143 & 0.2382 & 0.0285 & 0.3281 \\
\hline
\end{tabular}
\end{table}

\begin{figure}[H]
\begin{center}
\includegraphics[width=0.8\linewidth]{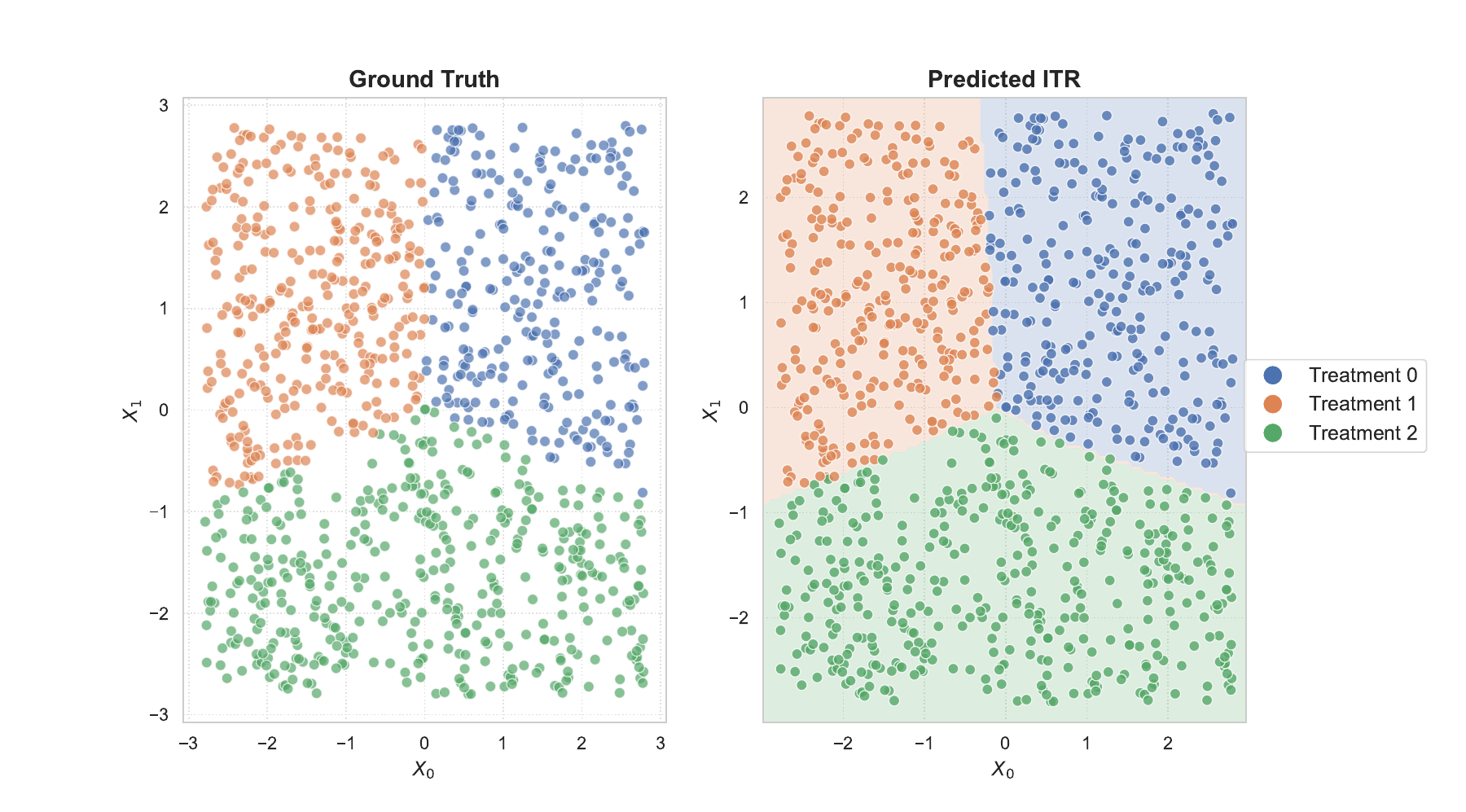}
\caption{Visualization of ground truth and predicted Individualized Treatment Rules with $c=(0,0)$.}
\label{fig8}
\end{center}
\end{figure} 

\begin{figure}[H]
\begin{center}
\includegraphics[width=0.8\linewidth]{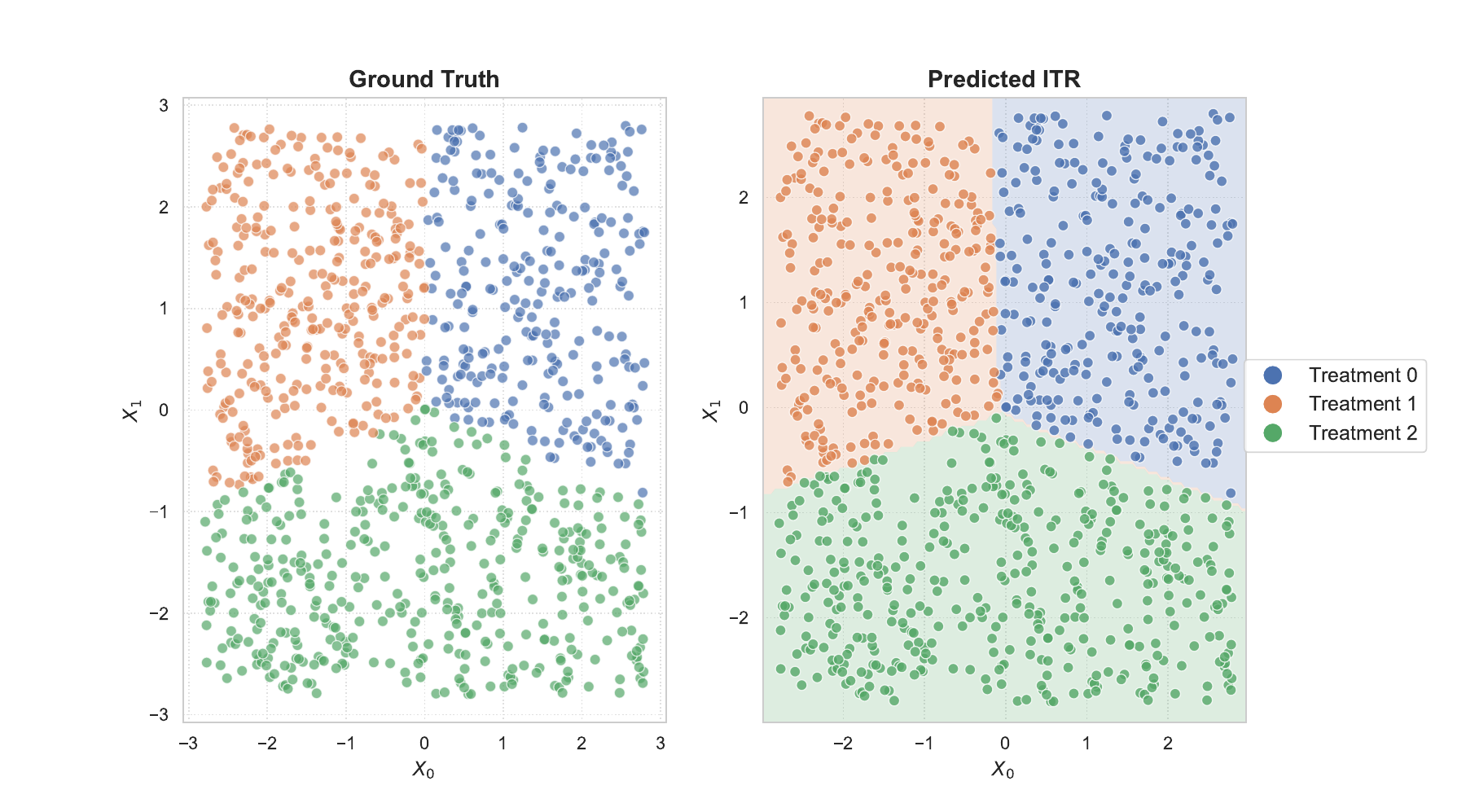}
\caption{Visualization of ground truth and predicted Individualized Treatment Rules with $c=(-1,1)$.}
\label{fig9}
\end{center}
\end{figure} 

\begin{figure}[H] 
\begin{center}
\includegraphics[width=0.8\linewidth]{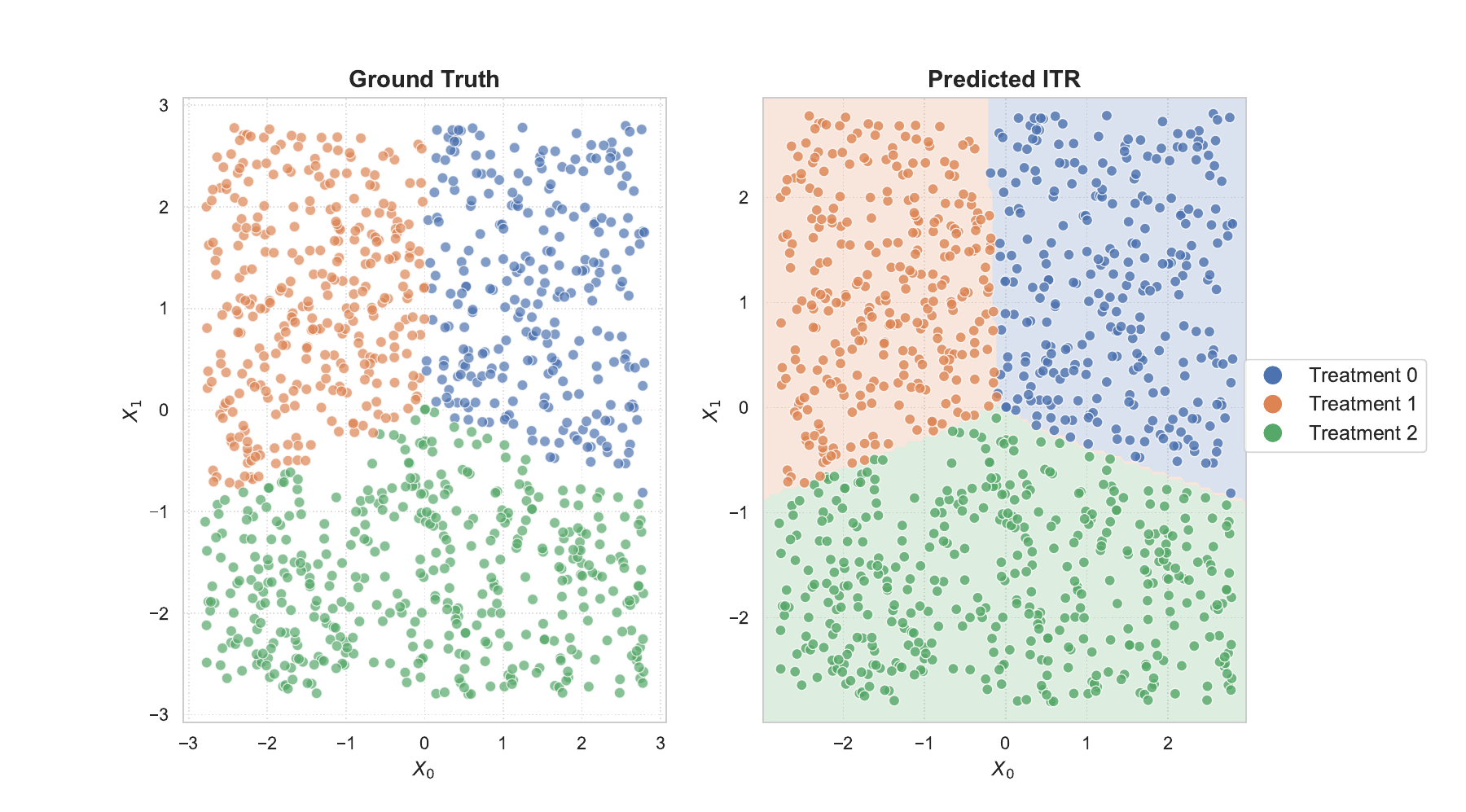}
\caption{Visualization of ground truth and predicted Individualized Treatment Rules with $c=(1,1)$.}
\label{fig10}
\end{center}
\end{figure} 

\begin{figure}[H]
\begin{center}
\includegraphics[width=0.8\linewidth]{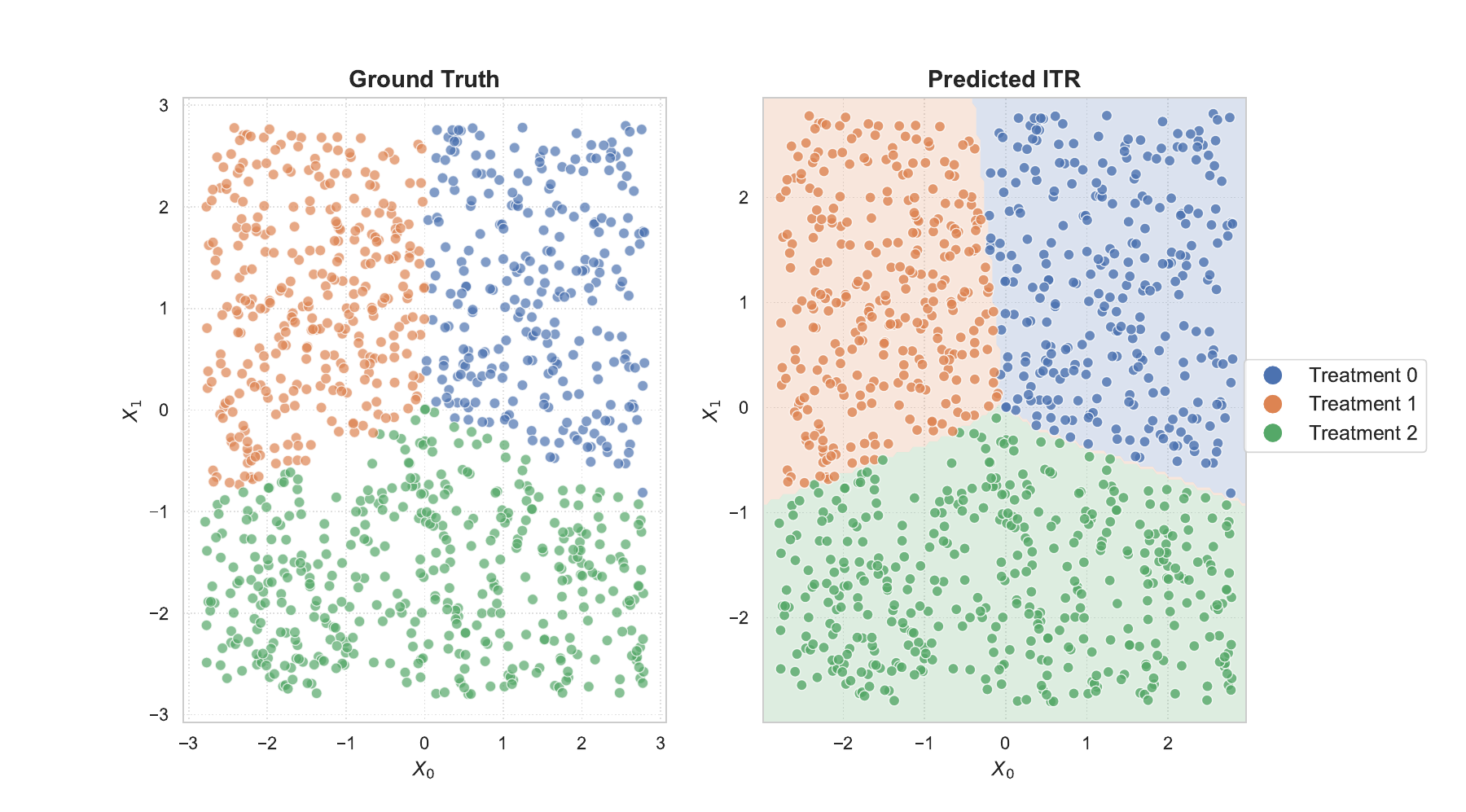}
\caption{Visualization of ground truth and predicted Individualized Treatment Rules with $c=(0,0)$.}
\label{fig5}
\end{center}
\end{figure} 

\begin{figure}[H]
\begin{center}
\includegraphics[width=0.8\linewidth]{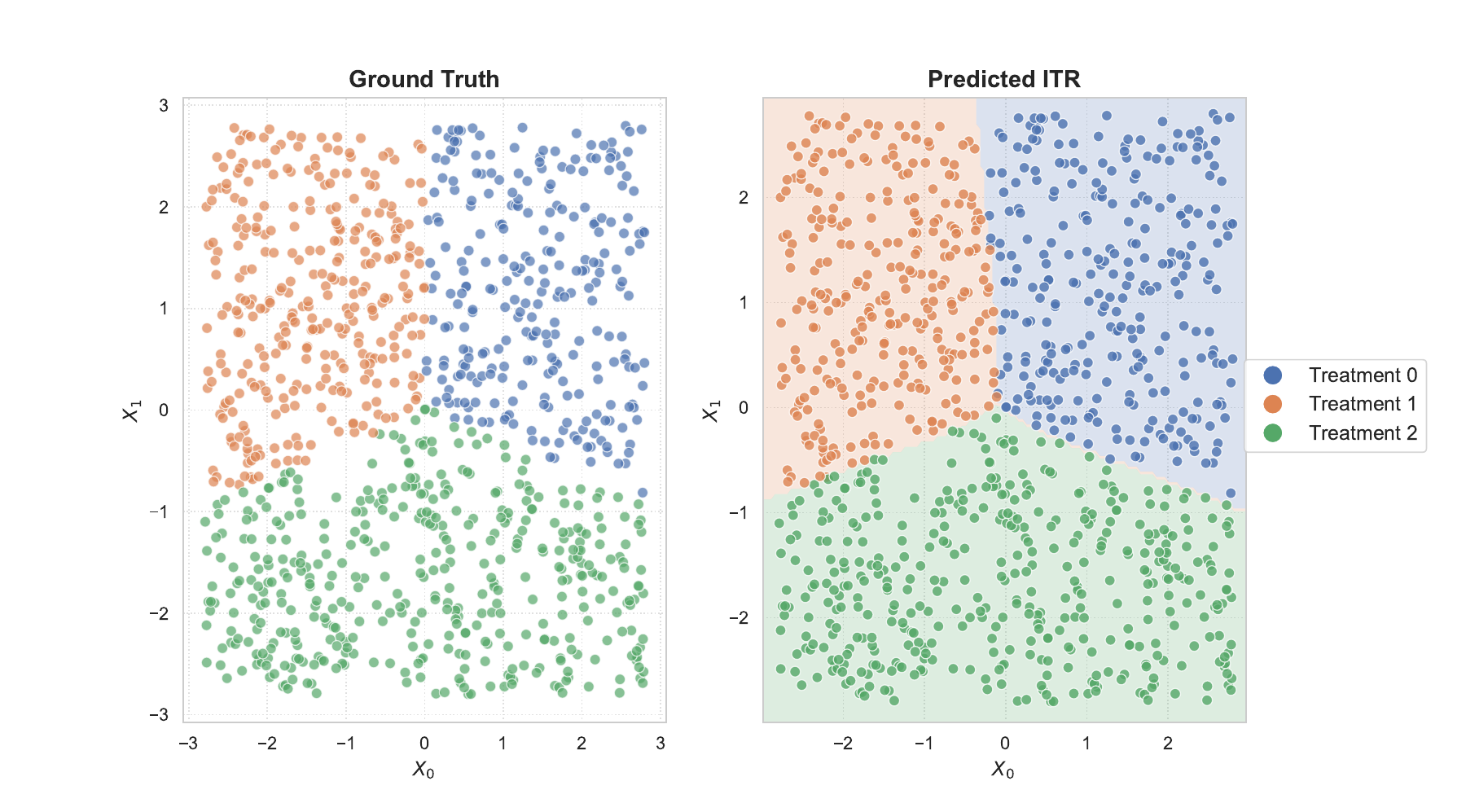}
\caption{Visualization of ground truth and predicted Individualized Treatment Rules with $c=(-1,1)$.}
\label{fig6}
\end{center}
\end{figure} 

\begin{figure}[H] 
\begin{center}
\includegraphics[width=0.8\linewidth]{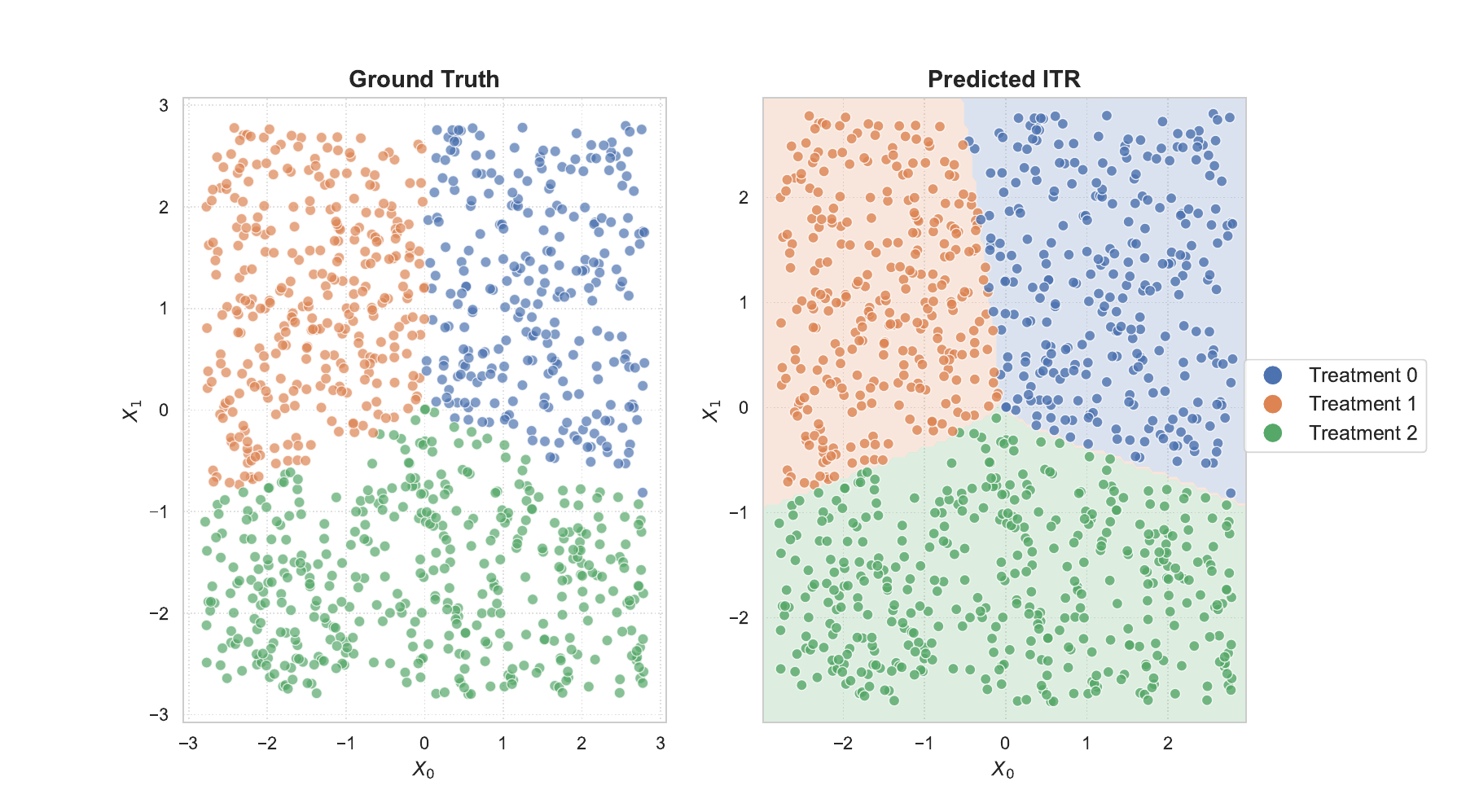}
\caption{Visualization of ground truth and predicted Individualized Treatment Rules with $c=(1,1)$.}
\label{fig7}
\end{center}
\end{figure} 

\subsection{Proofs of Theorems \ref{thm:consistency_beta} and \ref{thm:risk_consistency}} 

Define the population and empirical measures for any measurable function $g$ as $\mathbb Pg(O)=\int g(o)dP(o)$ and $\mathbb{P}_ng(O)=n^{-1}\sum_{i=1}^{n}g(O_{i})$.
Let $\mathbb{P}_{n_a} g(O)= n^{-1} \sum_{i=1}^{n} I(A=a) g(O_{i}) := n_a^{-1} \sum_{i=1}^{n_a} g(O_{i}) $.
Denote $a_n \lesssim b_n$ as $a_n \le cb_n$ for some $c>0$ and any $n$.

\textbf{Proof of Theorem \ref{thm:consistency_beta}}
\begin{proof}
Equation \ref{PLS2} implies that 
\begin{align*}
\hat{\beta}_{ja}^{APP} =& \bigg(\frac{1}{n_a}\sum_{i=1}^{n_a}\frac{\Delta_i}{\hat{G}_{ja}(Y_{ji},X_i)} X_i X_i^{\top} + \frac{c_1}{n_a}\sum_{i=1}^{n_a} X_iX_i^{\top} + \frac{c_2}{n-n_a}\sum_{i=1}^{n-n_a}  X_iX_i^{\top}  \bigg)^{-1} \\
&\cdot \bigg(\frac{1}{n_a}\sum_{i=1}^{n_a} X_i \Delta_i \hat{G}_{ja}^{-1}(Y_{ji},X_i) \log Y_{ji} + \frac{c_1}{n_a}\sum_{i=1}^{n_a} X_i \log \hat{T}_{ji}(a) + \frac{c_2}{n-n_a}\sum_{i=1}^{n-n_a} X_i \log \hat{T}_{ji}(a) \bigg).
\end{align*}
First, we can conclude that by Law of Large Numbers (LLN),
\begin{align}\label{EQ0}
\frac{c_1}{n_a}\sum_{i=1}^{n_a} X_iX_i^{\top} + \frac{c_2}{n-n_a}\sum_{i=1}^{n-n_a}  X_iX_i^{\top}  =(c_1+c_2) E[XX^{\top}] +o_p(1).
\end{align} 
Under Assumption \ref{ass:regularity}, when $\| \hat{G}_{ja} -G_{ja}\|_2 = o_p(1)$,
\begin{align}\label{EQ1}
&\frac{1}{n_a}\sum_{i=1}^{n_a}\frac{\Delta_i}{\hat{G}_{ja}(Y_{ji},X_i)} X_i X_i^{\top} - E[XX^{\top}]\nonumber  \\
=& \underbrace{\frac{1}{n_a}\sum_{i=1}^{n_a}\frac{\Delta_i}{\hat{G}_{ja}(Y_{ji},X_i)} X_i X_i^{\top}  - \frac{1}{n_a}\sum_{i=1}^{n_a}\frac{\Delta_i}{G_{ja}(Y_{ji},X_i)} X_i X_i^{\top}}_{A}  + \underbrace{\frac{1}{n_a}\sum_{i=1}^{n_a}\frac{\Delta_i}{G_{ja}(Y_{ji},X_i)} X_i X_i^{\top}- E[XX^{\top}]}_{B}.
\end{align}
Here $B = o_p(1)$ follows from LLN with Assumption \ref{ass:regularity} (ii) and (iii).
Meanwhile, for each element of the matrix $A$, $A_{kr}$, as the smallest eigenvalue of $E[XX^{\top}]$ is bounded away from zero, we have
\begin{align*}
|A_{kr}| \lesssim &  \bigg|\frac{1}{n_a}\sum_{i=1}^{n_a}  \Delta_i (G_{ja}(Y_{ji},X_i) -  \hat{G}_{ja}(Y_{ji},X_i) ) \bigg|\\
\lesssim & \| \hat{G}_{ja} -G_{ja}\|_2 + (\mathbb{P}_{n_a} - \mathbb{P})[ \hat{G}_{ja} -G_{ja} ] = o_p(1).
\end{align*}
Then, similarly, we have
\begin{align}\label{EQ2}
\frac{1}{n_a}\sum_{i=1}^{n_a} X_i \Delta_i \hat{G}_{ja}^{-1}(Y_{ji},X_i) \log Y_{ji} = E[X\Delta G_{ja}^{-1}(Y_{j},X) \log Y_j] + o_p(1),
\end{align}
and 
\begin{align}\label{EQ3}
\frac{c_1}{n_a}\sum_{i=1}^{n_a} X_i \log \hat{T}_{ji}(a) + \frac{c_2}{n-n_a}\sum_{i=1}^{n-n_a} X_i \log \hat{T}_{ji}(a) = (c_1+c_2) E[X f_j^{(a)}(X)]  + o_p(1),
\end{align}
where $f_j^{(a)}$ denotes the probability limit of $\hat{f}^{(a)}_{jn}$.

(i). When $c=(1,-1)$ or $c=(-1,1)$, combining \eqref{EQ0}-\eqref{EQ3}, we obtain that
\begin{align*}
\hat{\beta}_{ja}^{APP} = E[XX^{\top}]^{-1}E[X\Delta G_{ja}^{-1}(Y_{j},X) \log Y_j]  +o_p(1) = \beta_{ja} +o_p(1).
\end{align*}

(ii). When $f_j^{(a)} = f_{ja}^{(a)}$, we have $	\hat{\beta}_{ja}^{APP}= \beta_{ja} +o_p(1) $.
\end{proof}

\textbf{Proof of Theorem \ref{thm:risk_consistency}}
\begin{proof} 
We decompose the value into three terms:
\begin{align*}
0< V(d^*) - V(\sigma_S(\hat{h}_n)) =& V(\mathcal{D}_0) - \hat{V}(\mathcal{D}_0) + \hat{V}(\mathcal{D}_0) - \hat{V}(\sigma_S(\hat{h}_n)) + \hat{V}(\sigma_S(\hat{h}_n)) - V(\sigma_S(\hat{h}_n))\\
\leq &V(\mathcal{D}_0) - \hat{V}(\mathcal{D}_0)+\hat{V}(\sigma_S(\hat{h}_n)) - V(\sigma_S(\hat{h}_n)),
\end{align*} 
where the second inequality follows from the definition of $\hat{h}_n$.	

For any fixed function $h$, we first consider the $a$-th corresponding component of $\mathcal{ S}$, $S(t_1,t_2,X,a;\eta_{a})$, and have for given $a\in\mathcal{A}$, $\eta_a = (\beta_{1a}, \beta_{2a}, \gamma_{1a}, \gamma_{2a}, \theta_a)$ and $\hat{\eta}_a = (\hat{\beta}_{1a}, \hat{\beta}_{2a}, \hat{\gamma}_{1a}, \hat{\gamma}_{2a}, \hat{\theta}_a)$,
\begin{align*}
&\bigg| \mathbb{P} S(t_1,t_2,X,a;\eta_{a}) \frac{e^{h_a(X)}}{\sum_{i=0}^{K}e^{h_i(X)}} - \mathbb{P}_n S(t_1,t_2,X,a;\hat{\eta}_{a}) \frac{e^{h_a(X)}}{\sum_{i=0}^{K}e^{h_i(X)}}\bigg|\\
=&\underbrace{\bigg| \mathbb{P}_n S(t_1,t_2,X,a;\hat{\eta}_{a}) \frac{e^{h_a(X)}}{\sum_{i=0}^{K}e^{h_i(X)}} - \mathbb{P} S(t_1,t_2,X,a;\hat{\eta}_{a}) \frac{e^{h_a(X)}}{\sum_{i=0}^{K}e^{h_i(X)}}\bigg| }_{I_1}\\
&+ \underbrace{\bigg| \mathbb{P} S(t_1,t_2,X,a;\eta_{a}) \frac{e^{h_a(X)}}{\sum_{i=0}^{K}e^{h_i(X)}} - \mathbb{P} S(t_1,t_2,X,a;\hat{\eta}_{a}) \frac{e^{h_a(X)}}{\sum_{i=0}^{K}e^{h_i(X)}}\bigg|}_{I_2}.
\end{align*}
Here $I_1=o_p(1)$ due to Markov's inequality and the fact that $\eta_a$ is finite-dimensional.
By Taylor expansion, we have for $B$, 
\begin{align*}
I_2 \leq \big| \mathbb{P} \nabla S(t_1,t_2,X,a;\eta_{a})^{\top} (\eta_{a} - \hat{\eta}_a) \big| + O(\|\eta_{a} - \hat{\eta}_a)\|^2) \lesssim \|\eta_{a} - \hat{\eta}_a\|=o_p(1).
\end{align*}
Consequently, we can conclude that
$| V(d^*) - V(\sigma_S(\hat{h}_n)) |=o_p(1)$, which completes the proof.

\end{proof}
 
\end{document}